\documentclass[10pt]{ijnam}
\def\currentvolume{?}
\def\currentissue{?}
\def\currentyear{20??}
\def\month{March???}
\pagespan{1}{22}
\copyrightinfo{2025}{} 

\usepackage{amsmath}
\usepackage{url}            
\usepackage{booktabs}       
\usepackage{amsfonts}       
\usepackage{nicefrac}       
\usepackage{microtype}      
\usepackage{graphicx}
\usepackage{doi}
\usepackage{comment}
\usepackage{xcolor}
\usepackage{amsmath, amssymb}
\usepackage{amsthm}  
\newtheorem{theorem}{Theorem}[section]

\providecommand{\proofname}{Proof}

\begin{document}
	
	\title[Image Segmentation via Variational Model Based Tailored UNet: A Deep Variational Framework]
	{Image Segmentation via Variational Model Based Tailored UNet: A Deep Variational Framework}
	
	\author[K. QI \and Z. HUANG]{Kaili Qi \and Zhongyi Huang$^{*}$}
	\address{
		Department of Mathematical Sciences, Tsinghua University,
		Beijing, 100084, China
	}
	\email{qkl21@mails.tsinghua.edu.cn \and zhongyih@mail.tsinghua.edu.cn}
	\thanks{*Corresponding author}
	\author[W. YANG]{Wenli Yang}
	\address{
		 School of Mathematics, China University of Mining and Technology,
		  Xuzhou, 221116, China
	}
	\email{yangwl19@cumt.edu.cn}
	
	\date{January ?, ???? and, in revised form, March ??, ????.}
	

	\abstract{Traditional image segmentation methods, such as variational models based on partial differential equations (PDEs), offer strong mathematical interpretability and precise boundary modeling, but often suffer from sensitivity to parameter settings and high computational costs. In contrast, deep learning models such as UNet, which are relatively lightweight in parameters and unlimited by hardware and computing resources compared with recent deep neural networks, excel in automatic feature extraction but lack theoretical interpretability and require extensive labeled data. To harness the complementary strengths of both paradigms without milking the best possible engineering results, we propose Variational Model Based Tailored UNet (VM\_TUNet), a novel hybrid framework that integrates the fourth-order modified Cahn–Hilliard equation with the deep learning backbone of UNet, which combines the interpretability and edge-preserving properties of variational methods with the adaptive feature learning of neural networks. Specifically, a data-driven operator is introduced to replace manual parameter tuning, and we incorporate the tailored finite point method (TFPM) to enforce high-precision boundary preservation. Experimental results on benchmark datasets demonstrate that VM\_TUNet achieves superior segmentation performance compared to existing approaches, especially for fine boundary delineation.}
	
	\maketitle
	
\section{Introduction}
Image segmentation is dividing an image into distinct regions or segments to simplify its representation, making it easier to analyze or interpret. This technique is crucial in various fields, such as medical imaging for detecting tumors, autonomous driving for recognizing roads and obstacles, and video surveillance or scene analysis in computer vision. Traditional methods of image segmentation encompass approaches like thresholding, histogram-based bundling, region growing, $k\text{-means}$ clustering, watersheds, active contours, graph cuts, conditional and Markov random fields, and sparsity-based methods \cite{minaee2021image}. Despite their widespread use, traditional models face challenges when dealing with noise, complex backgrounds, and irregular object shapes. They rely on manually engineered features and mathematical models, such as variational techniques which utilize PDEs. 

Using variational energy minimization and PDE-based frameworks has a long history of addressing image segmentation problems and has achieved mickle admirable results \cite{chan2001active,liu2022two,yang2019image,zhu2013image}. Translate segmentation challenges into energy minimization problems, where the goal is to find an optimal segmentation function $u$ that minimizes a designed energy functional as
\begin{equation}\label{energy}
	E(u)=E_{\text{data}}(u;f)+E_{\text{reg}}(u),
\end{equation}
where $u$ represent the segmentation result, which can take one of several forms: a level set function, a region indicator function or a boundary curve, and $f$ denotes the input image. The first term of Eq.~(\ref{energy}) measures how well the segmentation matches the observed image characteristics and the second enforces smoothness constraints to prevent fragmented or irregular segmentation results. By solving this optimization problem, we can obtain a segmentation result that balances the internal similarity and edge smoothness of the image. Variational PDE-based models excel in image segmentation through their data efficiency, inherent noise robustness, computational lightness, interpretable energy minimization framework, topological adaptability, seamless physical prior integration, and unsupervised operation capability which have been proved in many aspects.

Despite their strong mathematical and physical foundations, PDE-based approaches can be challenging to apply effectively due to their sensitivity to initial setups and the requirement for manual parameter calibration. Manual parameter tuning is labor-intensive, and inappropriate settings, especially on noisy or textured images, can easily lead to under-segmentation or over-segmentation \cite{mumford1989optimal}. Like Chan-Vese model and other variants based on it, average intensities need manually set update format \cite{chan2001active,vese2002multiphase,zhang2010active}. With the surge of deep learning, its application in image segmentation has revolutionized the field, enabling unprecedented accuracy and efficiency in tasks such as medical imaging, autonomous driving, and remote sensing. Deep learning models, particularly parameter-efficient convolutional neural networks (CNNs) and architectures like FCN \cite{long2015fully}, UNet \cite{ronneberger2015u}, and so on, excel at automatically extracting hierarchical features from images, which eliminate the need for handcrafted features. This adaptability allows them to handle complex and diverse datasets with ease. Additionally, deep learning methods support end-to-end learning, reducing the reliance on manual parameter tuning and enabling scalable solutions. Their ability to generalize across various domains and achieve state-of-the-art performance makes them a powerful tool for image segmentation, despite challenges such as the need for large labeled datasets and computational resources.

In this study, we present \textbf{VM\_TUNet}, a novel approach that integrates deep learning with variational models for more effective image segmentation. By combining the Cahn-Hilliard equation with the UNet architecture, we create a hybrid model that overcomes the challenges of conventional methods, such as manual parameter tuning and sensitivity to initial conditions. Our approach benefits from the interpretability of variational models and the flexibility and scalability of deep learning, which provides a robust solution to complex image segmentation tasks. This work aims to provide a more adaptable and efficient model for diverse applications, particularly in fields like medical imaging and autonomous driving. Our contributions are threefold:
\begin{itemize}
	\item Integration of deep learning with traditional variational models: \\
	VM\_TUNet model proposed in this paper integrates traditional variational models with deep learning methods based on the UNet architecture, which combines the advantages of these two directions.
	\item Cahn-Hilliard equation to ensure boundary preservation: \\
	VM\_TUNet employs the fourth-order modified Cahn-Hilliard equation to ensure accurate boundary preservation and allow the model to be more adaptable to different types of data in complex scenarios.
	\item Tailored finite point method to compute accurately and effectively:  \\                
	VM\_TUNet utilizes TFPM to compute the {\color{blue} Laplacian }	operator during the segmentation process, which reduces parameter sensitivity and improves computational accuracy and efficiency.
\end{itemize}

	This paper is structured as follows: We introduce the variational model and deep learning model, especially the Cahn-Hilliard equation in Section \ref{Related work}. We propose VM\_TUNet with TFPM and UNet class in Section \ref{Method}, and demonstrate their effectiveness by comprehensive numerical experiments in Section \ref{Experiments}. This paper is concluded in Section \ref{Conclusion}.
	
\section{Related work}\label{Related work}
\subsection{Variational models}\label{chan}
Variational models play a crucial role in image segmentation by leveraging energy minimization and mathematical optimization techniques. Variational approaches to image segmentation commonly involve evolving active contours to precisely delineate the boundaries of objects \cite{caselles1997geodesic,chan2001active,kass1988snakes}. In addition to curve evolution techniques, numerous region-based image segmentation models exist, with the Mumford-Shah variational model being one of the most renowned which has many variants such as Chan-Vese model. Moreover, the use of Euler's elastica is prevalent in the development of variational methods for mathematical imaging \cite{bae2017augmented,shen2003euler,tai2011fast,zhu2013image}. Variational approaches offer transparency and interpretability, which make them suitable for medical image analysis. However, their performance can be highly dependent on initialization and often requires manual parameter tuning, which limits the generalization of these methods and they typically processes only one image at a time, which makes them inefficient for relatively large-scale tasks. 

The Chan-Vese segmentation algorithm is designed to segment objects without clearly defined boundaries. This algorithm is based on level sets that are evolved iteratively to minimize an energy, which is defined by weighted values corresponding to the sum of differences intensity from the average value outside the segmented region, the sum of differences from the average value inside the segmented region, and a term which is dependent on the length of the boundary of the segmented region. This algorithm was first proposed by Tony Chan and Luminita Vese, which  is based on the Mumford–Shah functional and aims to segment an image by finding a contour $C$ that best separates the image into regions of approximately constant intensity. It seeks to minimize the following energy functional
\[
\begin{split}
	E(c_1, c_2, C) &= \mu \cdot \text{Length}(\partial C)\\
	& + \lambda_1 \int_{\text{inside}(C)} |I(x, y) - c_1|^2 \, dxdy 
	+ \lambda_2 \int_{\text{outside}(C)} |I(x, y) - c_2|^2 \, dxdy,
\end{split}
\]
where $I(x,y)$ is the input image; $c_1,c_2$ are average intensities inside and outside the contour; $\mu,\lambda_1,\lambda_2$ are positive weighting parameters; and $\partial C$ is the contour boundary. To handle topology changes, like contour splitting, the contour $C$ is represented implicitly using a level set function $\phi(x,y)$ such that
\begin{equation}\notag
	\begin{cases}
		C=\{(x,y)\in\Omega:\phi(x,y)=0\},\\
		\text{insdie}(C)=\{(x,y)
		\in\Omega:\phi(x,y)>0\},\\
		\text{outside}(C)=\{(x,y)\in\Omega:\phi(x,y)<0\},
	\end{cases}  
\end{equation}
where $\Omega$ is the image domain. Using the Heaviside function $H(\phi)$, the energy can be rewritten in the level set form
\begin{equation}\notag
	\begin{split}
		&E(c_1, c_2, \phi) = \mu \int_{\Omega} \delta(\phi(x, y)) |\nabla \phi(x, y)| \, dxdy \\
		+ 
		&\lambda_1 \int_{\Omega} |I(x, y) - c_1|^2 H(\phi(x, y)) \, dxdy + 
		\lambda_2 \int_{\Omega} |I(x, y) - c_2|^2 \big(1 - H(\phi(x, y))\big) \, dxdy,
	\end{split}
\end{equation}
where $H(\phi)$ is the Heaviside function and $\delta(\phi)$ is the Dirac delta approximation. At each iteration, the optimal values of $c_1$ and $c_2$ (region averages) are 
\[
c_1 = \frac{\int_{\Omega} I(x, y)H(\phi(x, y)) \, dxdy}{\int_{\Omega} H(\phi(x, y)) \, dxdy}
,\quad
c_2 = \frac{\int_{\Omega} I(x, y)\big(1 - H(\phi(x, y))\big) \, dxdy}{\int_{\Omega} (1 - H(\phi(x, y))) \, dxdy}.
\]
To minimize the energy, the level set function $\phi$ evolves according to the Euler–Lagrange equation, leading to the following PDE
\[
\frac{\partial \phi}{\partial t} = \delta(\phi) \left[ \mu \cdot \text{div}\left( \frac{\nabla \phi}{|\nabla \phi|} \right) 
- \lambda_1 \big(I(x, y) - c_1\big)^2 + \lambda_2 \big(I(x, y) - c_2\big)^2 \right],
\]
where the first term is a regularization term (curvature) and the second and third terms are data fidelity forces, pushing the contour to regions where $I$ is closer to $c_1$ or $c_2$. The Dirac delta function $\delta(\phi)$ is approximated as
\[
\delta_\epsilon(\phi) = \frac{1}{\pi} \cdot \frac{\varepsilon}{\phi^2 + \varepsilon^2},
\]
and the Heaviside function is approximated as
\[
H_\epsilon(\phi) = \frac{1}{2} \left(1 + \frac{2}{\pi} \arctan\left( \frac{\phi}{\varepsilon} \right) \right).
\]
In \cite{chan2001active}, authors proposed a numerical algorithm using finite differences to solve this problem. In the Chan-Vese image segmentation algorithm, the contour appears dense and scattered at the beginning because of the initialization of the level set function. Typically, the level set function is initialized with a shape (like a circle or square) that defines the starting contour. Around the zero level set, there may be many values close to zero due to discretization and numerical approximation, which can lead to the appearance of multiple contour lines when visualized. Additionally, some implementations deliberately initialize multiple contours to help the algorithm explore the image structure more effectively. As the algorithm iterates, these contours evolve and gradually merge into a single, smooth boundary that accurately segments the object of interest.

\begin{figure}
	\centering\includegraphics[width=1\linewidth]{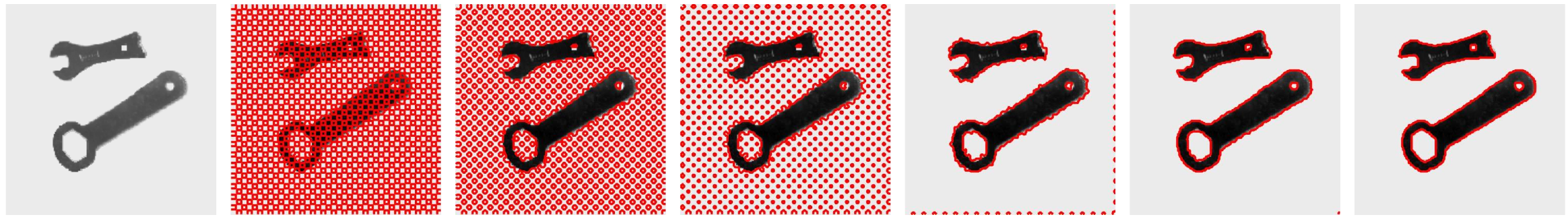}  
	\caption{Traditional image segmentation using variational PDE methods involves formulating an energy functional whose minimization corresponds to an optimal segmentation, and solving the associated partial differential equations to find this minimum. Here are the results of different number of iterations of the Chan-Vese algorithm.}   
\end{figure}

\subsection{Deep learning models}

The development of deep learning has revolutionized image segmentation, which enables outstanding performance in areas such as medical imaging, autonomous driving, and remote sensing. FCN and similar models have significantly advanced the use of deep learning architectures in the field of image semantic segmentation. Leading models including UNet, SegNet \cite{badrinarayanan2017segnet}, UNet++ \cite{Zhou2018UNetAN}, DeepLabV3+ \cite{Chen_2018_ECCV} and Segment Anything model \cite{kirillov2023segment,zhang2024segment} that excel in tasks requiring high accuracy and scalability, which leverage large datasets and powerful computational resources. In contrast to traditional variational models, deep learning provides greater flexibility in handling complex data, automates feature extraction, and supports end-to-end learning, which eliminates the need for manual parameter adjustment. Deep learning methods, however, often necessitate large-scale labeled datasets and significant computational resources, and their complex, non-intuitive nature can reduce interpretability. Overall, while deep learning dominates modern image segmentation due to its performance and flexibility, variational models remain valuable for specific applications where interpretability and theoretical rigor are prioritized.

Inspired by the advantages of both traditional mathematical models and data-driven methods, there is potential to incorporate robust mathematical principles and tools into neural networks for image segmentation tasks to get deep learning-enhanced variational models \cite{celledoni2021structure,lu2018beyond,ruthotto2020deep,weinan2017proposal}. A direct approach to achieving this integration involves incorporating functions derived from various PDE models into the loss function \cite{chen2019learning,kim2019cnn,le2018reformulating}. These techniques, called loss-inserting methods, tackle the challenge of numerous PDEs lacking defined energy properties and not being representable as gradient flows of explicit energy. Researchers have also explored modifying neural network architectures to improve their interpretability \cite{chen2015learning,cheng2019darnet,lunz2018adversarial,marcos2018learning,ruthotto2020deep}. For example, Liu et al. design Inverse Evolution Layers (IELs) by simulating the inverse of evolution processes governed by partial differential equations. Instead of promoting desirable properties. IELs act as bad property amplifiers and provide a principled and interpretable way to incorporate mathematical priors from PDEs into deep learning models, offering a powerful tool for enhancing segmentation performance and robustness, and emphasize undesirable characteristics, such as noise or concavity, in the network outputs during training, effectively penalizing them and guiding the model to produce more desirable results \cite{liu2025inverse}. Nevertheless, introducing physics-informed loss and regularization would greatly increase the difficulty of model training.

Recently, in \cite{AAM-39-4,tai2024pottsmgnet}, theoretical and practical connections between operator-splitting methods and deep neural networks demonstrate their synergistic application in improving image segmentation tasks through efficient optimization and enhanced model interpretability, which presents a rigorous mathematical framework that explains the architecture of encoder-decoder convolutional neural networks through the lens of optimization and control theory. PottsMGNet bridges the gap between deep learning and mathematical optimization, providing a principled explanation for encoder-decoder networks using the Potts model and multigrid-based numerical schemes. Based on the split Bregman algorithm for the Potts model, PottsNN integrates a total variation regularization term from Fields of Experts and parameterizes the penalty parameters and thresholding function of the model making manual settings trainable, which is a physics- and optimization-inspired neural network that unifies the strengths of PDE modeling and deep learning for explainable and effective image segmentation \cite{cui2024trainable,wang2024pottsnn}. After PottsMGNet, Doublewell Nets bridge the extension of the Allen-Cahn type Merriman-Bence-Osher scheme and neural networks to solve the Potts model \cite{liu2024double}. Doublewell Net is a deep neural network framework inspired by the doublewell potential model, which is commonly used in physics and variational image processing to model binary segmentation problems. The key idea behind Doublewell Net is to design a network whose architecture and dynamics mimic the gradient flow of an energy functional that includes a doublewell potential, encouraging the output to converge toward two distinct phases. However, the above models are all low-order while the higher-order Cahn-Hilliard equation can better maintain the boundary.

\subsection{Cahn-Hilliard equation}

The Cahn-Hilliard equation was originally proposed by Cahn and Hilliard and describes the phase separation that occurs when a mixture of two substances is quenched into an unstable state \cite{cahn1958free}. The Cahn-Hilliard equation is the gradient flow of the generalized Ginzberg Landau proper energy functional under the $H^{-1}$ norm. Van der Waals first proposed the generalized Ginzberg-Landau free energy functional, which accurately describes the mixing energy of two substances
\begin{equation}\notag
	E(u)=\int_{\Omega}\dfrac{\varepsilon^2}{2}|\nabla u|^2+W(u)\mathrm{d}x,
\end{equation}
where $u$ represents the concentration of one of the species, the concentration of the other species is $1-u$, $W(u)$ is the doublewell function $W(u)=u^2(1-u)^2$ or Lyapunov functional $W(u)=\frac{1}{4}(u^2-1)^2$ and the parameter $\varepsilon$ controls the interface between the two metals.

The space $H^{-1}$ is the zero-mean dual subspace of $H^1$, that is, for any given $v\in H^{-1},\int_{\Omega}v(x)dx=0$ holds if and only if $v=\Delta\phi,\int_{\Omega}\phi dx=0$. Define the inner product on $H^{-1}$ as
\begin{equation}\notag
	<v_1,v_2>_{H^{-1}} = <\nabla\phi_{1},\nabla\phi_2>_{L^2},
\end{equation}
where $v_1=\Delta\phi_1$, $v_2=\Delta\phi_2$. The Gateaux derivative of $E(u)$ is
\begin{equation}\notag
	E^{\prime}(u)=-\varepsilon^2\Delta u+W^{\prime}(u).
\end{equation}
Thereby,
\begin{equation}\notag
	\begin{split}
		\frac{\partial}{\partial \lambda}E(u+\lambda v)\big|_{\lambda=0}&=\int_{\Omega}(-\varepsilon^2\Delta u+W^{\prime}(u))\Delta\phi \mathrm{d}x\\
		&=<-\nabla(-\varepsilon^2\Delta u+W^{\prime}(u)),\nabla\phi>_{L^2}\\
		&=<-\Delta(-\varepsilon^2\Delta u+W^{\prime}(u)),\Delta\phi>_{H^{-1}}\\
		&=<-\Delta(-\varepsilon^2\Delta u+W^{\prime}(u)),v>_{H^{-1}},		
	\end{split}
\end{equation}
that is,  $\nabla_{H^{-1}}E(u)=-\Delta(-\varepsilon^2\Delta u+W^{\prime}(u))$. The directional derivative of $E$ at $u$ with respect to direction $v$ is $<\nabla_{H^{-1}}E(u),v>_{H^{-1}}$, along direction $v=-\nabla_{H^{-1}}E(u)$, the directional derivative $<\nabla_{H^{-1}}E(u),v>_{H^{-1}}$ is negative \
 and $\left|<\nabla_{H^{-1}}E(u),v>_{H^{-1}}\right|$ is the largest. This direction is called the direction of fastest descent, so the gradient flow is 
\begin{equation}\notag
	\frac{\partial u}{\partial t}=-\nabla_{H^{-1}}E(u)=-\Delta(\varepsilon^2\Delta u-W^{\prime}(u)).
\end{equation}Cahn-Hilliard equation is important in science and industry, whose fourth-order term ensures smooth boundaries and noise robustness, especially in image segmentation tasks. The corresponding modified Cahn-Hilliard model has been proposed in the study of the diffusion of droplets on solid surfaces and the repulsion and competition between biological populations. Scholars systematically studied the Cahn-Hilliard equation and have established mathematical theories about it. 

Here, we consider the modified Cahn-Hilliard model. Given image $f$ and the corresponding segmented image $u$, we consider the following Cahn-Hilliard type equation for binary image segmentation
\begin{equation}\label{modified}
	u_{t}=-\Delta\big(\varepsilon_{1}\Delta u-\frac{1}{\varepsilon_{2}}W^{\prime}(u)\big)-\{\lambda_{1}(f-c_{1})^2-\lambda_{2}(f-c_{2})^2\}\frac{\varepsilon_{3}}{\pi\{\varepsilon_{3}^2+(u-\frac{1}{2})^2\}},
\end{equation}
which can be simplified to the following minimization problem as
\begin{equation}\notag
	\begin{split}
		\min\limits_{u}\{E(u;c_1,c_2):&=\int_{\Omega}\big(\frac{\varepsilon_1}{2}|\nabla u|^2+\frac{1}{\varepsilon_{2}}W(u)\big)\mathrm{d}x\\
		&+\lambda_{1}\int_{\{u\geq\frac{1}{2}\}}(f-c_1)^2\mathrm{d}x+\lambda_{2}\int_{\{u<\frac{1}{2}\}}(f-c_2)^2\mathrm{d}x\},
	\end{split}
\end{equation}
where $\varepsilon_1,\varepsilon_2,\varepsilon_3,\lambda_1,\lambda_2>0$, $u$ satisfies $\partial u/\partial n=\partial \Delta u/\partial n=0$ on $\partial\Omega$, $W(u)=u^2(u-1)^2$, and $c_1$, $c_2$ are two constants that can be assigned using some strategy. These parameters depend on manual experience and affect the results strongly. For example, when the value range of image $f$ is $[0,1]$, we can initially set $c_1=1$, $c_2=0$, then solve the steady-state solution of Eq.~\eqref{modified} and update $c_1$, $c_2$ according to the following formula 
\begin{equation}\notag
	c_1=\frac{\int_{\Omega}\{\frac{1}{2}+\frac{1}{\pi}\arctan\big(\frac{u-\frac{1}{2}}{\varepsilon_{3}}\big)\}f\mathrm{d}x}{\int_{\Omega}\{\frac{1}{2}+\frac{1}{\pi}\arctan\big(\frac{u-\frac{1}{2}}{\varepsilon_3}\big)\}\mathrm{d}x},\ 
	c_2=\frac{\int_{\Omega}\{\frac{1}{2}-\frac{1}{\pi}\arctan\big(\frac{u-\frac{1}{2}}{\varepsilon_{3}}\big)\}f\mathrm{d}x}{\int_{\Omega}\{\frac{1}{2}-\frac{1}{\pi}\arctan\big(\frac{u-\frac{1}{2}}{\varepsilon_3}\big)\}\mathrm{d}x},
\end{equation} 
where the aforementioned model can also be extended for color image segmentation, as demonstrated in \cite{yang2019image}.

However, the above strategy suffers from parameter sensitivity, lack of adaptability, and heavy reliance on prior knowledge, leading to inconsistent results and high computational costs in complex scenarios. Furthemore, the implementation of modified complex Cahn-Hilliard equations facilitates multi-phase segmentation in \cite{wang2022multi}.

Define
\begin{equation}
	F(f)=[\lambda_{1}(f-c_{1})^2-\lambda_{2}(f-c_{2})^2]\frac{\varepsilon_{3}}{\pi[\varepsilon_{3}^2+(u-\frac{1}{2})^2]},
\end{equation}
where inspired by the approximation theory of deep neural network, we will use UNet class architecture to represent $F(f)$ as a subnetwork to avoid manual setting and adjustment of parameters like Doublewell Net \cite{Bao2023ApproximationAO,liu2024double}, {\color{blue} and in each subsequent step, the Sigmoid function is used to ensure that $F(f)$ is bounded}. Then we have to solve the following equation
\begin{equation}\label{F}
	u_{t}=-\Delta\big(\varepsilon_{1}\Delta u-\frac{1}{\varepsilon_{2}}W^{\prime}(u)\big)-F(f).
\end{equation}
	
	\section{Method}\label{Method}
	\subsection{Adaption to deep learning framework}
	For convenience, we convert Eq.~\eqref{F} into two coupled second-order parabolic equations
	\begin{equation}\label{tc}
		\begin{aligned}
			v &= \varepsilon_{1}\Delta u-\frac{1}{\varepsilon_{2}}W^{\prime}(u),\\
			u_t&=-\Delta v-F(f),
		\end{aligned}
	\end{equation}
	where $u$ and $v$ satisfy $\partial u/\partial n=\partial v/\partial n=0$. For spatial discretization, we use spatial steps $\Delta x_1=\Delta x_2=h$ for some $h>0$ and simultaneously, let $\tau$ be the time step, for $n\geq 0,t^n=n\tau$, then we have the outcome $u^n=u(t^n)$ and $v^n=v(t^n)$ at every substep. Set $u_0=\mathrm{Sig}(W^0*f+b^0)$ where $\mathrm{Sig}$ is the sigmoid function and $W^0$, $b^0$ are the convolution kernel and bias. We propose using a convolutional layer followed by a sigmoid function to generate an initial condition $u^0$ and then solve Eq.~\eqref{tc} until a finite time $t=T$ and use $u(x,T)$ as the final segmentation result.
	
	Thus, we can obtain
	\begin{equation}\label{v0}
		\begin{cases}
			v=\varepsilon_{1}\Delta u-\dfrac{1}{\varepsilon_{2}}W^{\prime}(u) ,& \text{in}\quad\Omega\times(0,T], \\
			\dfrac{\partial v}{\partial n}=0,	& \text{on}\quad\partial\Omega,\\[0.6em]
			v^0=\varepsilon_{1}\Delta u_0-\dfrac{1}{\varepsilon_{2}}W^{\prime}(u_0), & \text{in}\quad\Omega,
		\end{cases}
	\end{equation}
	and
	\begin{equation}\label{u0}
		\begin{cases}
			\dfrac{\partial u}{\partial t}=-\Delta v-F(f),&\text{in} \quad\Omega\times(0,T],\\[0.5em]
			\dfrac{\partial u}{\partial n}=0,&\text{on} \quad\partial\Omega,\\[0.5em]
			u^0=u_0,&\text{in}\quad\Omega.
		\end{cases}
	\end{equation}

	For $v$ at every substep, we directly compute it as the following scheme from Eq.~\eqref{v0}
	\begin{equation}\label{v1}
		\begin{cases}
			v^n=\varepsilon_{1}\Delta u^n-\dfrac{1}{\varepsilon_{2}}W^{\prime}(u^n),&\text{in}\quad\Omega, \\
			\dfrac{\partial v^n}{\partial n}=0	,&\text{on}\quad\partial\Omega,\\
			v^0=\varepsilon_{1}\Delta u_0-\dfrac{1}{\varepsilon_{2}}W^{\prime}(u_0), &\text{in}\quad\Omega,
		\end{cases}
	\end{equation}
	and for $u$ at every substep, we use a one-step forward Euler scheme to time discretize Eq.~\eqref{u0}
	\begin{equation}\label{u1}
		\begin{cases}
			u^{n+1}=u^n-\tau \Delta v^n-\tau F(f),&\text{in}\quad\Omega,\\
			\dfrac{\partial u^{n+1}}{\partial n}=0,&\text{on}\quad\partial\Omega,\\
			u^0=u_0,&\text{in}\quad\Omega,
		\end{cases}
	\end{equation}
		
{\color{blue}
	
	\begin{theorem}[Stability Estimate]
		Assume that $W'(u^n)$ is L-Lipschitz continuous, the region force term satisfies $\|F(u^n)\| \le M_F$ and $	\|\Delta u^{n+1}-\Delta u^n\|^2 \le C_\Delta$ based on the practical needs of image processing problems here, where $M_F$ and $C_\Delta$ are constants. 
		Let $u^0 \in H^2(\Omega)$. Then there exist positive constants $A,\ B,\ C,\ D$ independent of $n$ such that
		\[
		A\|u^{n+1}\|^2 + B\|\Delta u^{n+1}\|^2
		\le
		D\left(\|u^n\|^2 + \|\Delta u^n\|^2\right) + C,
		\]
	\end{theorem}
	
	\begin{proof}
		
		Consider the scheme
		\begin{equation}\notag
			\begin{cases}
				v^n = \varepsilon_1 \Delta u^n - \dfrac{1}{\varepsilon_2} W'(u^n), \\[0.5em]
				u^{n+1} = u^n - \tau \Delta v^n - \tau F(u^n).
			\end{cases}
		\end{equation}
		Eliminating $v^n$ gives
		\begin{equation}\label{eq:eliminate}
			u^{n+1} =
			u^n
			-\tau\varepsilon_1\Delta^2 u^n
			+
			\frac{\tau}{\varepsilon_2}\Delta(W'(u^n))
			-
			\tau F(u^n).
		\end{equation}
		Taking the $L^2$ inner product of \eqref{eq:eliminate} with $u^{n+1}$ and adding $\tau\varepsilon_1\|\Delta u^{n+1}\|^2$ to both sides yields
		\begin{equation}\notag
			\begin{aligned}
				\|u^{n+1}\|^2+\tau\varepsilon_1\|\Delta u^{n+1}\|^2
				=&
				\langle u^n,u^{n+1}\rangle
				-\tau\varepsilon_1\langle\Delta u^n,\Delta u^{n+1}\rangle
				+\tau\varepsilon_1\|\Delta u^{n+1}\|^2
				\\
				&+
				\frac{\tau}{\varepsilon_2}\langle W'(u^n),\Delta u^{n+1}\rangle
				-\tau\langle F(u^n),u^{n+1}\rangle .
			\end{aligned}
		\end{equation}
		
		By
		$
		2\langle a-b,a\rangle=\|a\|^2+\|a-b\|^2-\|b\|^2
		$
		with $a=\Delta u^{n+1}$ and $b=\Delta u^n$, we obtain
		\begin{equation}\notag
			\begin{aligned}
				\|u^{n+1}\|^2
				+
				\frac{\tau\varepsilon_1}{2}\|\Delta u^{n+1}\|^2
				=
				&
				\langle u^n,u^{n+1}\rangle
				-
				\frac{\tau\varepsilon_1}{2}\|\Delta u^n\|^2	-
				\tau\langle F(u^n),u^{n+1}\rangle 
				\\
				&
				+
				\frac{\tau\varepsilon_1}{2}\|\Delta u^{n+1}-\Delta u^n\|^2
				+
				\frac{\tau}{\varepsilon_2}\langle W'(u^n),\Delta u^{n+1}\rangle.
			\end{aligned}
		\end{equation}
		Utilizing Cauchy--Schwarz inequality and Young inequality with $W''(u^n) \le K$ and $\|F(u^n)\| \le M_F$,
		
		\[
		\langle u^n,u^{n+1}\rangle
		\le
		\frac12\|u^n\|^2+\frac12\|u^{n+1}\|^2 ,
		\]
		
		\[
		-\tau\langle F(u^n),u^{n+1}\rangle
		\le
		\frac{\tau}{2\delta}\|u^{n+1}\|^2
		+
		\frac{\tau\delta}{2}M_F^2,
		\]
		
		\[
		\frac{\tau}{\varepsilon_2}\langle W'(u^n),\Delta u^{n+1}\rangle
		\le
		\frac{\tau L\gamma}{2\varepsilon_2}(1+\|u^n\|^2)
		+
		\frac{\tau L}{2\varepsilon_2\gamma}\|\Delta u^{n+1}\|^2 .
		\]
		
		\[
		\frac{\tau\varepsilon_1}{2}\|\Delta u^{n+1}-\Delta u^n\|^2
		\le
		\frac{\tau\varepsilon_1}{2}C_\Delta .
		\]
Collecting the estimates yields
		\begin{equation}\notag
			A\|u^{n+1}\|^2
			+
			B\|\Delta u^{n+1}\|^2
			\le
			D\left(\|u^n\|^2
			+
			\|\Delta u^n\|^2\right)
			+
			C,
		\end{equation}
where
		
		\[
		A = \frac12-\frac{\tau}{2\delta},\ 
		B = \frac{\tau\varepsilon_1}{2}-\frac{\tau L}{2\varepsilon_2\gamma},\ D = \frac12+\frac{\tau L\gamma}{2\varepsilon_2},\ C =
		\frac{\tau L\gamma}{2\varepsilon_2}
		+
		\frac{\tau\varepsilon_1}{2}C_\Delta
		+
		\frac{\tau\delta}{2}M_F^2.
		\]
Choosing $\delta>\tau$ and $\gamma>\dfrac{L}{\varepsilon_1\varepsilon_2}$ ensures $A>0$ and $B>0$, which completes the proof.
		
	\end{proof}
	
}
	
	Suppose now that we are given a training set of images $\{f_{i}\}_{i=1}^{I}$ with their foreground-background segmentation masks $\{g_{i}\}_{i=1}^{I}$ and we will learn a data-driven operator $F$ so that for any given image $f$ with similar properties as the training set, the steady state of Eq.~\eqref{F} is close to its segmentation $g$. Denote $\Theta$ as the collection of all parameters to be determined from the data, i.e., the parameters in $F$ and then $u(x,T)$ from Eq.~\eqref{u1} only depends on $f$ and $\Theta$. Furthermore, we will determine $\Theta$ by solving
	\begin{equation}\notag
		\min\limits_{\Theta}\frac{1}{I}\sum\limits_{i=1}^{I}\ell(u(x,T;\Theta,f_{i}),g_{i}),
	\end{equation}
	where $\ell(\cdot,\cdot)$ is a loss function which could be the loss functional including hinge loss, logistic loss, and $L^2$ norm, measuring the differences between its arguments \cite{rosasco2004loss}.
	
	\subsection{Variational Model Based
		Tailored UNet}
	\begin{figure}[ht]
		\centering\includegraphics[width=1\linewidth]{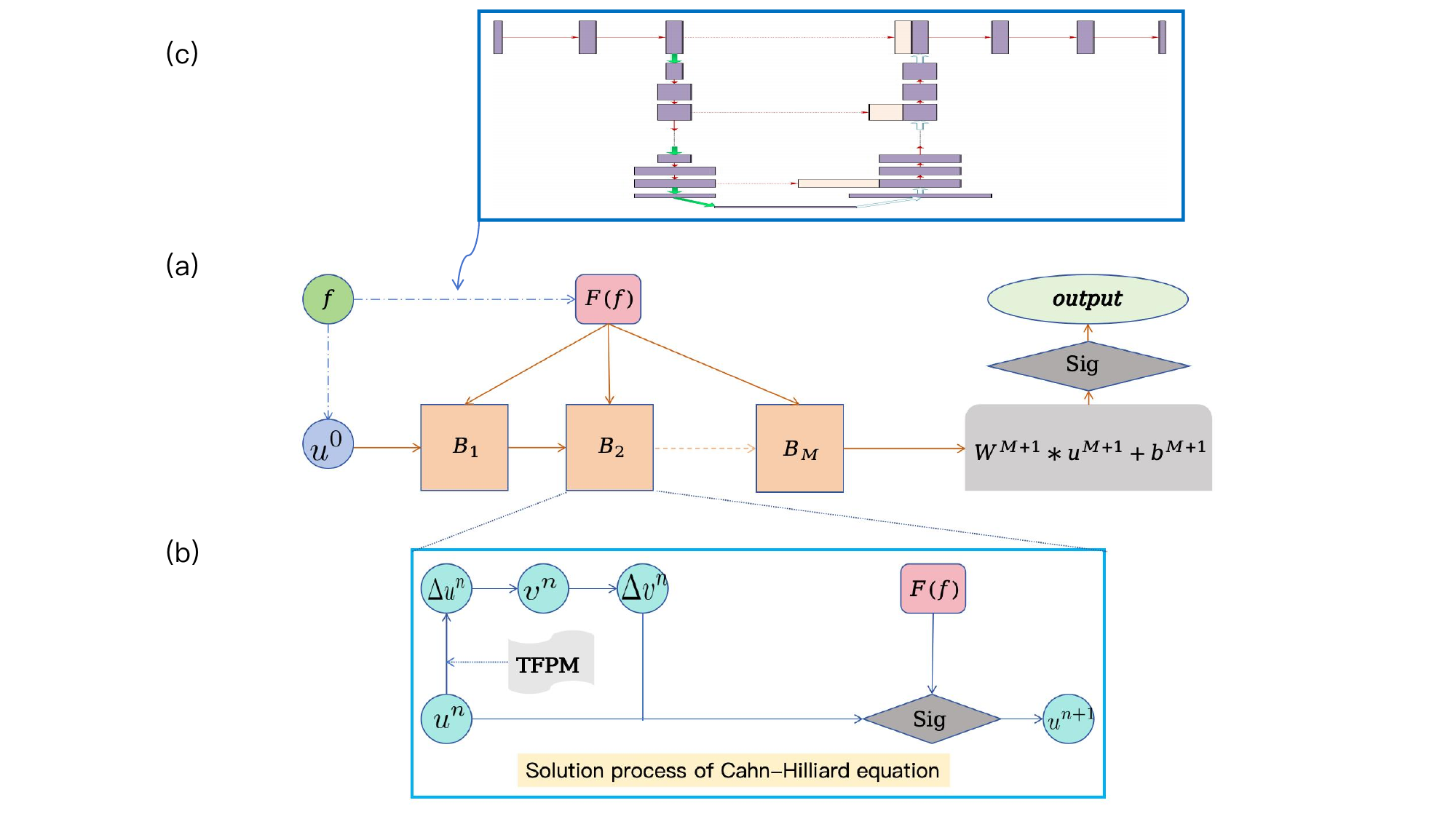}  
		\caption{VM\_TUNet architecture. (a) VM\_TUNet process. (b) VM\_TUNet block, which is the process of solving Cahn-Hilliard equation. (c) UNet class architecture which approximates $F(f)$ in (a).}  
		\label{VMTUNet illustration}  
	\end{figure}
	Assume $f$ is an image of size $N_1\times N_2\times D$, then we approximate $F(f)$ in Eq.~\eqref{u1} by a UNet class neural network specified by a channel vector $\boldsymbol{c}$. We call the procedure $u^n\rightarrow v^n\rightarrow u^{n+1}$ a VM\_TUNet block, denoted by $B_{n+1}$; see Figure~\ref{VMTUNet illustration}(b) for an illustration with activation $\mathrm{Sig}$, where $B_{n+1}$ contains trainable parameters of $F(f)$. The input for $B_{n+1}$ includes the output of the previous block $u^n,\Delta v_{n}$ and $F(f)$. There are $M$ VM\_TUNet blocks, $F(f)$ is passed to all VM\_TUNet blocks, and $u^0$ is passed through every VM\_TUNet block sequentially. Denoting the $m$-th VM\_TUNet block by $B_{m}$, and after $B_M$, we add a convolution layer followed by a sigmoid function. Denote the kernel and bias in the last convolution layer by $W^{M+1}$ and $b^{M+1}$, respectively, then we have VM\_TUNet as the following formula
	\begin{equation}\notag
		P(f)=\mathrm{Sig}(W^{M+1}*(B_M(\cdots B_2(B_1(u^0,F(f)),F(f)),F(f))+b^{M+1}),
	\end{equation}
	where we show the architecture of VM\_TUNet in Figure \ref{VMTUNet illustration}.
	
	PottsMGNet and Doublewell Nets derived from the Potts model and operator-splitting method give a clear explanation for the encode-decode type of neural network like UNet. However, the mathematical treatment and splitting treatment of PottsMGNet and Doublewell Net II (DN-II) are different from those of Doublewell Net I (DN-I) and VM\_TUNet here. In addition, DN-I and VM\_TUNet fix $F(f)$ over time and assume $F(f)$ is only a function of the input image $f$ \cite{liu2024double}, but DN-II is not. In the meantime, VM\_TUNet uses a higher-order Cahn-Hilliard equation model so that the sharp boundaries are well preserved during image segmentation, compared to the previous methods.
	
	\paragraph{\textbf{Accommodate periodic boundary conditions}}
	Periodic boundary conditions are often considered in image segmentation \cite{liu2024double}. Let $\Omega=[0,L_1]\times[0,L_2]$ and denote the two spatial directions by $x$ and $y$. To accommodate the periodic boundary conditions, we replace Eq.~\eqref{v1} and Eq.~\eqref{u1} by 
	\begin{equation}\label{v2}
		\begin{cases}
			v^n=\varepsilon_{1}\Delta u^n-\dfrac{1}{\varepsilon_{2}}\big(4(u^n)^3-6(u^n)^2+2u^n\big),&\text{in}\quad\Omega, \\
			v^n(0,y)=v^n(L_1,y),&0\leq y\leq L_2,\\[0.5em]
			v^n(x,0)=v^n(x,L_2),&{ }0\leq x\leq L_1,\\
			v^0=\varepsilon_{1}\Delta u_0-\dfrac{1}{\varepsilon_{2}}\big(4(u_0)^3-6(u_{0})^2+2u_{0}\big),&\text{in}\quad\Omega,
		\end{cases}
	\end{equation}
	and 
	\begin{equation}\label{u2}
		\begin{cases}
			u^{n+1}=u^n-\tau \Delta v^n-\tau F(f),&\text{in}\quad\Omega,\\
			u^{n+1}(0,y)=u^{n+1}(L_1,y),&0\leq y\leq L_2,\\
			u^{n+1}(x,0)=u^{n+1}(x,L_2),&0\leq x\leq L_1,\\
			u^0=u_0,&\text{in}\quad\Omega,
		\end{cases}
	\end{equation}
	respectively.
	
	\paragraph{\textbf{Tailored finite point method}}
	For $\Delta u^n$ in Eq.~\eqref{v2}, we propose using the tailored finite point method (TFPM) method to compute it \cite{Han2010TailoredFP}. Let
	\begin{equation}\notag
		u(x,y)=c_{0}+c_1e^{-\lambda x}+c_{2}e^{\lambda x}+c_{3}e^{-\lambda y}+c_4e^{\lambda y},
	\end{equation}
	by linearizing $W^{\prime}(u)$ we have
	\begin{equation}\notag
		W^{\prime}(u)=4u^3-6u^2+2u=(4u^2+2)u-6u^2,
	\end{equation}
	where $4u^2+2$ and $-6u^2$ can be seen as slice constants \cite{yang2019image}, then we can obtain 
	\begin{equation}\notag
		\lambda=\sqrt{\dfrac{4(u_{i,j})^{2}+2}{\varepsilon_{1}\varepsilon_2}},\  c_0=\dfrac{6(u_{i,j})^{2}}{4(u_{i,j})^{2}+2},
	\end{equation}
	from
	\begin{equation}\notag
		\varepsilon_1\Delta u=\dfrac{1}{\varepsilon_2}(4(u_{i,j})^{2}+2)u-\dfrac{6}{\varepsilon_2}(u_{i,j})^{2},
	\end{equation}
	where $u_{i,j}^n=u^n(ih,jh)$ on discretized mesh points.
	
	Moreover, by solving
	\begin{equation}\notag
		\begin{aligned}
			c_{1}e^{-\lambda h}+c_{2}e^{\lambda h}+c_{3}+c_{4}+c_{0}=u^{n}_{i+1,j},\\c_{1}+c_{2}+c_{3}e^{\lambda h}+c_{4}e^{-\lambda h}+c_{0}=u^{n}_{i,j+1},\\c_{1}e^{\lambda h}+c_{2}e^{-\lambda h}+c_{3}+c_{4}+c_{0}=u^{n}_{i-1,j},\\
			c_{1}+c_{2}+c_{3}e^{-\lambda h}+c_{4}e^{\lambda h}+c_{0}=u^{n}_{i,j-1},\\
		\end{aligned}
	\end{equation}
	we have
	\begin{equation}\notag
		\Delta u^n=\lambda^2 \dfrac{u^{n}_{i+1,j}+u^{n}_{i,j+1}+u^{n}_{i-1,j}+u^{n}_{i,j-1}-4c_0}{4\cosh^2({\lambda h/2})}.
	\end{equation}
	
	For $\Delta v^n$ in Eq.~\eqref{u2}, it is the {\color{blue} Laplacian } of $v^n$ approximated by central difference, which is realized by convolution $W_{\Delta}*v^n$ with
	\begin{equation}\label{laplace}
		W_{\Delta}=\frac{1}{h^2}
		\begin{bmatrix}
			0 & 1 & 0\\
			1 & -4 & 1 \\
			0 & 1 & 0
		\end{bmatrix}.
	\end{equation}
	
	\subsection{UNet class}
	The UNet class is designed to capture multiscale features of images, which features a symmetric encoder-decoder structure with skip connections to combine high-resolution features from the encoder with upsampled features from the decoder. For the resolution levels in the encoding and decoding parts, ordered from finest to coarsest, we represent their corresponding number of channels as a vector $\boldsymbol{c}=[c_1,\dots,c_S]$, where each element $c_s(s=1,\dots,S)$ is a positive integer, and $S$ denotes the total number of resolution levels. For an image input with size $N_1\times N_2\times D$, the structure of such a class is illustrated in Figure~\ref{UNet architecture}. Furthermore, UNet class has the advantage of light parameters compared to other deep learning methods such as Segment Anything model (SAM), so we prefer it for approximation of $F(f)$.
	
	\begin{figure}[ht]
		\centering\includegraphics[width=1.0\linewidth]{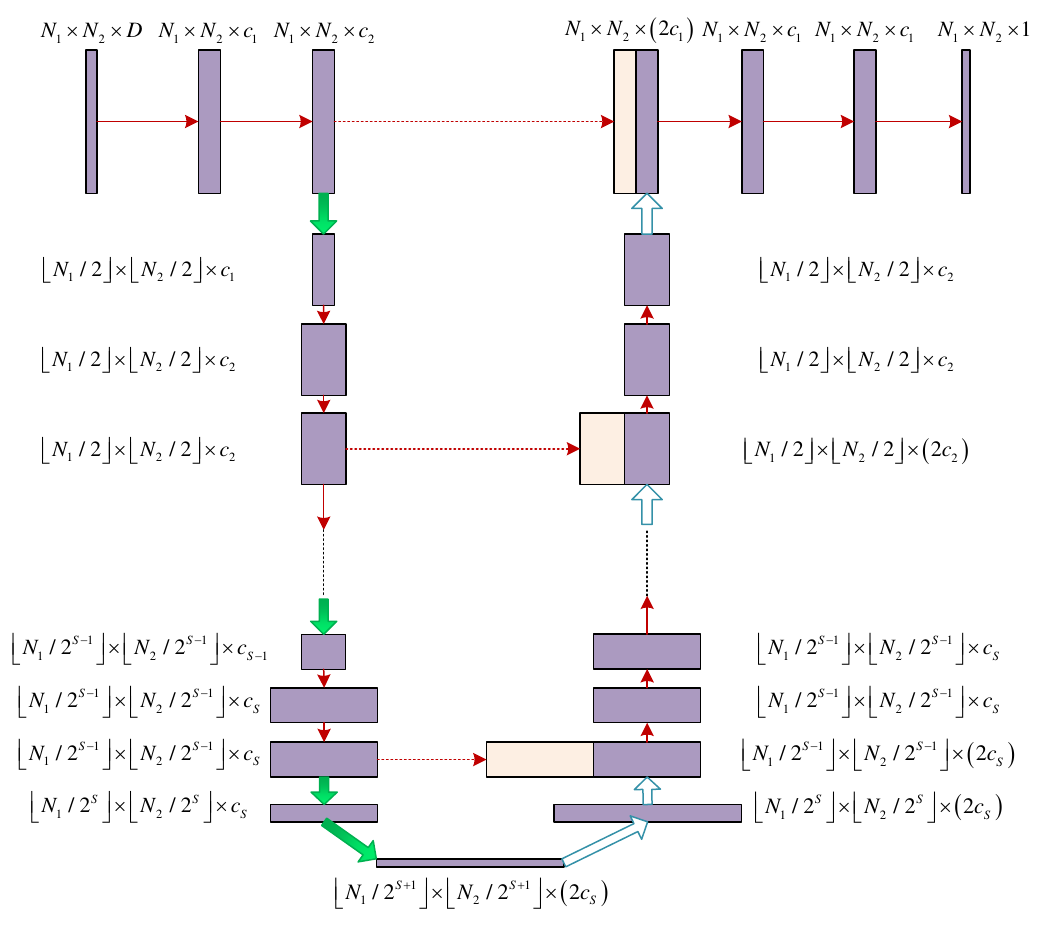}
		\caption{Illustration of UNet type network with input of size $N_1\times N_2\times D$. The left branch is the encoding part, the right branch is the decoding part, and the bottom rectangle denotes the bottleneck. Green arrows represent downsampling operations. Transparent arrows represent upsampling operations. Horizontal red dashed arrows represent skip connections. The orange rectangles denote the outputs of the encoding part that are passed to the decoding part via the skip connections. The length and width of the rectangle represent the output resolution and number of channels, respectively. Every network in this class can be fully characterized by the channels vector $\boldsymbol{c}$: given a $\boldsymbol{c}\in\mathbb{R}^S$, the corresponding network has $S+1$ resolution levels, $c_s$ channels at resolution level $s$ for $1\leq s\leq S$, and $2c_S$ channels at resolution level $S+1$ \cite{liu2024double}.}
		\label{UNet architecture}
	\end{figure}

	The architecture is designed to capture multiscale features of images: each resolution level corresponds to features of one scale. For the original UNet, it has $\boldsymbol{c}=[64,128,256,512]$. The general UNet type architecture used here for image segmentation is inspired by the original UNet, which features a symmetric encoder-decoder design with skip connections that bridge high-resolution features from the contracting path to the expanding path. This structure enables precise localization while maintaining semantic context, which makes UNet especially effective in segmentation tasks with limited training data. The architecture follows this foundational UNet paradigm, which maintains the characteristic downsampling and upsampling paths connected via skip connections. This modification enhances the representational robustness of the latent features, encouraging a more distinct and stable embedding of segmentation-relevant structures. While still falling under the umbrella of "UNet-like" architectures, VM\_TUNet extends the traditional UNet by embedding a task-driven regularization objective into its core, which contributes to improved generalization and sharper segmentation boundaries—especially in challenging or ambiguous regions of the input.
	
	\subsection{Lightweight architecture}
	In this paper, we only modified the PDE equation and calculation method, and did not change the original architecture of UNet. The number of parameters in our model is not much different from that in DN-I, and we fully utilize the advantage of the UNet network architecture being lightweight compared to other deep neural networks, such as the Segment Anything model (SAM), Transformer and so on. The comparison of parameters of various deep learning methods in image segmentation is shown in Table~\ref{table3}.
	
	\begin{table}[ht]
		\caption{Comparison of the number of parameters of DN-I, VM\_TUNet with UNet, UNet++, DeepLabV3+, and TransUNet, SAM on ECSSD (M: Represents one million parameters).}
		\label{table3}
		\centering
		\begin{tabular}{ccccccc}
			\toprule
			\multicolumn{7}{c}{Number of parameters on ECSSD} \\
			\midrule
			DN-I & VM\_TUNet & UNet & UNet++ & DeepLabV3+ & TransUNet & SAM \\
			7.7M & 7.8M & 31.0M & 35.0M & 40.0M & 105.0M & \textgreater{} 600.0M \\
			\bottomrule
		\end{tabular}
	\end{table}
	
	In certain tasks, especially those involving limited computational resources or small-scale datasets, it is often preferable to use lightweight architectures like UNet over large-scale models such as the SAM. UNet offers a highly efficient encoder-decoder structure with significantly fewer parameters, making it ideal for applications like medical image segmentation or embedded systems where inference speed and memory usage are critical. While SAM demonstrates strong generalization and zero-shot capabilities, its massive size and resource demands make it less suitable for scenarios requiring fast, cost-effective deployment or fine-tuning on small domain-specific datasets. Another advantage of UNet is that it does not require pretraining on large-scale datasets, unlike many modern deep learning models that rely on extensive pretraining. UNet is designed to perform well even when trained from scratch on relatively small datasets, thanks to its symmetric architecture and skip connections that preserve spatial information effectively. This makes it particularly suitable for domains like biomedical imaging, where annotated data is scarce and domain-specific features differ significantly from natural images.
	
	\section{Experiments}\label{Experiments}
	In this section, we evaluate VM\_TUNet on four semantic segmentation datasets: The Extended Complex Scene Saliency Dataset (ECSSD) \cite{7182346}, the Retinal Images Vessel Tree Extraction Dataset (RITE) \cite{hu2013automated}, Hong Kong University-Saliency Dataset (HKU-IS) \cite{LiYu15} and Dalian University of Technology-OMRON Saliency Dataset (DUT-OMRON) \cite{yang2013saliency}. We train with Adam optimizer and 600 epochs for ECSSD and RITE, and 800 epochs for HKU-IS and DUT-OMRON. In the VM\_TUNet model, without specification, the UNet architecture is employed with a channel configuration of $\boldsymbol{c}=[128,128,128,128,256]$. {\color{blue} The model consists of 10 blocks, and the parameters are set as follows: $\varepsilon_{1}=1$, $\varepsilon_{2}=1$ and $\tau=0.5$, where the model is robust within a reasonable range of parameters, but extreme values can still affect convergence and segmentation quality.} We implement all numerical experiments on a single NVIDIA RTX 4090 GPU. The PyTorch code of DN-I is
	available at \url{https://github.com/liuhaozm/Double-well-Net}. In our experiments, UNet, UNet++ and DeepLabV3+ are implemented by using the Segmentation Models PyTorch package \cite{Iakubovskii:2019}. The comparison of running time of different methods above on ECSSD is shown in Table \ref{table1}. 
	
	\begin{table}[ht]
		\caption{Comparison of running time of UNet, UNet++, DeepLabV3+, DN-I, and VM\_TUNet on ECSSD (s: Represents seconds per epoch).}
		\label{table1}
		\centering
		\begin{tabular}{ccccccc}
			\toprule
			\multicolumn{5}{c}{Running time per epoch on ECSSD} \\
			\midrule
			UNet & UNet++ & DeepLabV3+ & DN-I & VM\_TUNet \\
			7.72s & 8.51s & 9.18s & 11.32s & 11.81s \\
			\bottomrule
		\end{tabular}
	\end{table}
	
	{\color{blue}
		As shown in Tables 1 and 2 of the paper, the parameter counts of DN-I and VM\_TUNet are indeed lower than that of UNet. However, DN-I and VM\_TUNet adopt an operator splitting scheme that involves both linear and nonlinear operations. In practice, each block represents one iterative solution step of the original equation, which consumes additional computational time.
	}
	\subsection{Comparison with other networks}
	We compare the proposed model with corresponding state-of-the-art image segmentation convolutional neural networks mainly including UNet \cite{ronneberger2015u}, UNet++ \cite{Zhou2018UNetAN}, DeepLabV3+ \cite{Chen_2018_ECCV} and DN-I \cite{liu2024double}. In these models, the output is generated by passing the final layer through a sigmoid activation function, resulting in a tensor where each element falls within the range of $[0,1]$. To convert this into a binary segmentation map, a threshold value $T$ is applied to the output matrix as
	\begin{equation}
		T\circ P(f)=
		\begin{cases}
			1\quad \text{if } P(f)\geq 0.5,\\
			0\quad \text{if } P(f)<0.5.
		\end{cases}
	\end{equation}
	
	The similarity between the predicted output of the model and the provided ground truth mask is evaluated using two metrics, accuracy and the dice score:
	\begin{equation}
		\text{accuracy}=\frac{1}{K}\sum\limits_{k=1}^{K}\left[\dfrac{|[T\circ P(f_{k})]\cap g_{k}|}{N_{1}N_{2}}\times 100\%\right],
	\end{equation}
	and
	\begin{equation}
		\text{dice score}=\frac{1}{K}\sum\limits_{k=1}^{K}\left[\dfrac{2|[T\circ P(f_{k})]\cap g_{k}|}{|T\circ P(f)|+|g|}\right],
	\end{equation}
	where $|g|$ denotes the number of nonzero elements of a binary function $g$ and $\cap$ is the logic ``and" operation.
	
	The chosen images are intended to visually demonstrate the distinctions between the proposed approach and currently available CNN architectures. The predictions generated by the proposed method closely align with the ground truth masks, whereas competing models exhibit inaccuracies, either by incorrectly segmenting certain objects or failing to capture specific regions entirely, especially at the borders. Particularly, for the cicada image segmentation task in the first row of Figure~\ref{HKU1}, whose ground truth does not have antennae, but our method makes it. To measure the complexity of different models, we present the outcome of loss, accuracy, and dice score in Table~\ref{table2}, and we can see that VM\_TUNet outperforms other models.
	
	\begin{table}
		\caption{Comparison of the Accuracy and Dice score of VM\_TUNet with UNet, UNet++, DeepLabV3+, and DN-I  on ECSSD, RITE, HKU-IS and DUT-OMRON.}
		\label{table2}
		\centering
		\resizebox{0.98\textwidth}{!}{
			\begin{tabular}{lllllllll}
				\toprule
				&\multicolumn{2}{c}{ECSSD}&    \multicolumn{2}{c}{RITE}  &
				\multicolumn{2}{c}{HKU-IS}&
				\multicolumn{2}{c}{DUT-OMRON}\\
				\cmidrule(r){2-3}  \cmidrule(r){4-5} \cmidrule(r){6-7} \cmidrule(r){8-9}&  Accuracy   & Dice score &Accuracy&Dice score&Accuracy&Dice score&Accuracy&Dice score \\
				\midrule
				\multicolumn{1}{c}{UNet}& \multicolumn{1}{c}{0.894$\pm$0.003}& \multicolumn{1}{c}{0.868$\pm$0.001}& \multicolumn{1}{c}{0.930$\pm$0.004}& \multicolumn{1}{c}{0.678$\pm$0.002}&
				\multicolumn{1}{c}{0.904$\pm$0.003}&
				\multicolumn{1}{c}{0.871$\pm$0.001}&
				\multicolumn{1}{c}{0.881$\pm$0.001}&
				\multicolumn{1}{c}{0.859$\pm$0.002}\\ 
				\multicolumn{1}{c}{UNet++}& \multicolumn{1}{c}{0.900$\pm$0.004}& \multicolumn{1}{c}{0.879$\pm$0.001}& \multicolumn{1}{c}{0.937$\pm$0.005}& \multicolumn{1}{c}{0.682$\pm$0.001}&
				\multicolumn{1}{c}{0.908$\pm$0.001}&
				\multicolumn{1}{c}{0.877$\pm$0.002}&
				\multicolumn{1}{c}{0.889$\pm$0.002}&
				\multicolumn{1}{c}{0.863$\pm$0.001}\\ 	
				\multicolumn{1}{c}{DeepLabV3+}& \multicolumn{1}{c}{0.910$\pm$0.003}& \multicolumn{1}{c}{0.885$\pm$0.002}& \multicolumn{1}{c}{0.942$\pm$0.002}& \multicolumn{1}{c}{0.695$\pm$0.001}&
				\multicolumn{1}{c}{\textbf{0.923}$\pm$0.002}&
				\multicolumn{1}{c}{0.887$\pm$0.002}&
				\multicolumn{1}{c}{0.892$\pm$0.001}&
				\multicolumn{1}{c}{0.868$\pm$0.002}\\ 
				\multicolumn{1}{c}{DN-I}& \multicolumn{1}{c}{0.921$\pm$0.002}& \multicolumn{1}{c}{0.889$\pm$0.001}& \multicolumn{1}{c}{0.941$\pm$0.002}& \multicolumn{1}{c}{0.702$\pm$0.002}&
				\multicolumn{1}{c}{0.910$\pm$0.003}&
				\multicolumn{1}{c}{0.881$\pm$0.001}&
				\multicolumn{1}{c}{0.890$\pm$0.004}&
				\multicolumn{1}{c}{0.871$\pm$0.001}\\
				\multicolumn{1}{c}{VM\_TUNet}& \multicolumn{1}{c}{\textbf{0.937}$\pm$0.002}& \multicolumn{1}{c}{\textbf{0.892}$\pm$0.002}& \multicolumn{1}{c}{\textbf{0.948}$\pm$0.001} & \multicolumn{1}{c}{\textbf{0.713}$\pm$0.002} &
				\multicolumn{1}{c}{0.921$\pm$0.002}&
				\multicolumn{1}{c}{\textbf{0.890}$\pm$0.001}&
				\multicolumn{1}{c}{\textbf{0.903}$\pm$0.002}&
				\multicolumn{1}{c}{\textbf{0.878}$\pm$0.002}\\ 
				\bottomrule
		\end{tabular}}
	\end{table}
	
	\paragraph{\textbf{ECSSD}}
	ECSSD is a semantic segmentation data set containing 1000 images with complex backgrounds and manually labeled masks, we resize all images to the size 256$\times$256, using 800 images for training and 200 images for testing and the partial results of ECSSD is shown in Figure \ref{crane-tiger}.
	
	\begin{figure}[ht]
		\centering\includegraphics[width=1\linewidth]{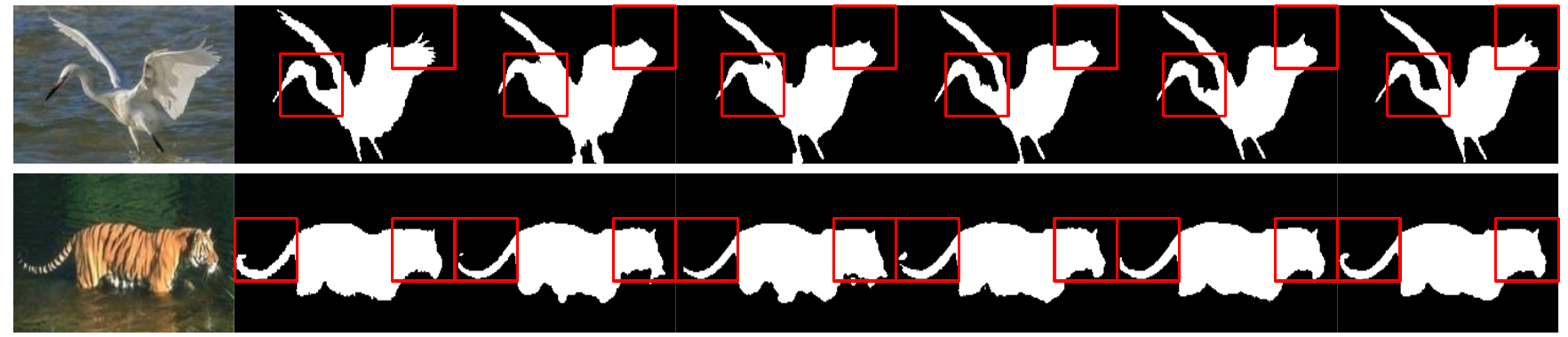}  
		\caption{Comparison results of crane (Above) and tiger (Below) image of ECSSD between the proposed model and UNet, UNet++, DeepLabV3+, and DN-I. The pictures from left to right are: Image; Ground Truth; and the results of  UNet, UNet++, DeepLabV3+, DN-I, and VM\_TUNet, respectively.}  
		\label{crane-tiger}  
	\end{figure}
	
	\paragraph{\textbf{RITE}}
	For RITE, which is a dataset for the segmentation and classification of arteries and veins on the retinal fundus containing 40 images, we resize all images to 256$\times$256, using 20 images for training and 20 images for testing and the partial results of RITE is shown in Figure \ref{left-right}. 
	
	\begin{figure}[ht]	\centering\includegraphics[width=1\linewidth]{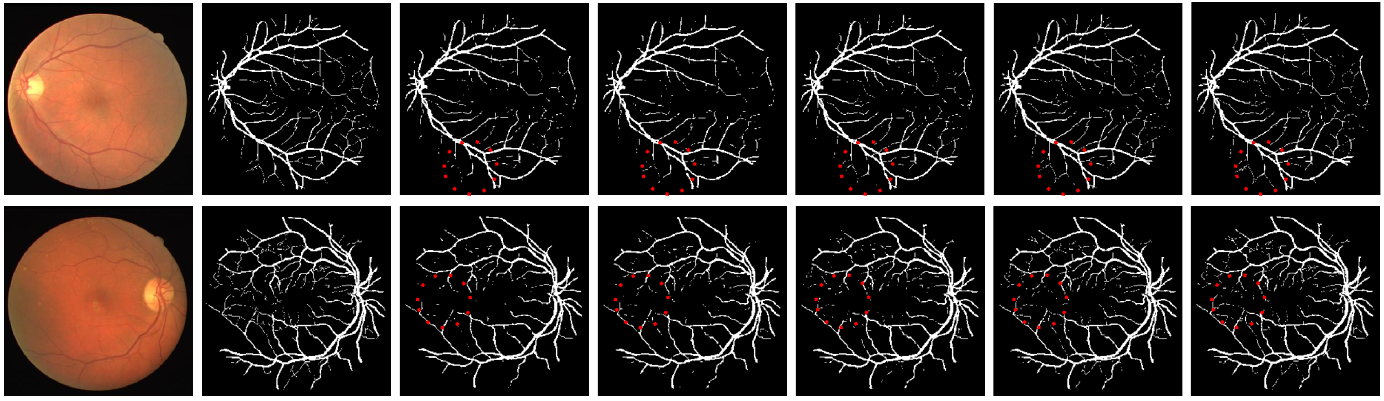}  
		\caption{Comparison results of left (Above) and right (Below) retinal of  RITE between the proposed models: UNet, UNet++, DeepLabV3+, DN-I, and VM\_TUNet. The pictures from left to right are: Image; Ground Truth; and the results of  UNet, UNet++, DeepLabV3+, DN-I, and VM\_TUNet, respectively.}  
		\label{left-right}  
	\end{figure}
	
	\paragraph{\textbf{HKU-IS}}
	HKU-IS is a visual saliency prediction dataset which contains 4447 challenging images, most of which have either low contrast or multiple salient objects. For HKU-IS, we resize all images to the size 256$\times$256, using 3000 images for training and 1447 images for testing. The partial results of HKU-IS is shown in Figure~\ref{HKU1} and Figure~\ref{HKU2}.
	
	\begin{figure}[ht]	\centering\includegraphics[width=1\linewidth]{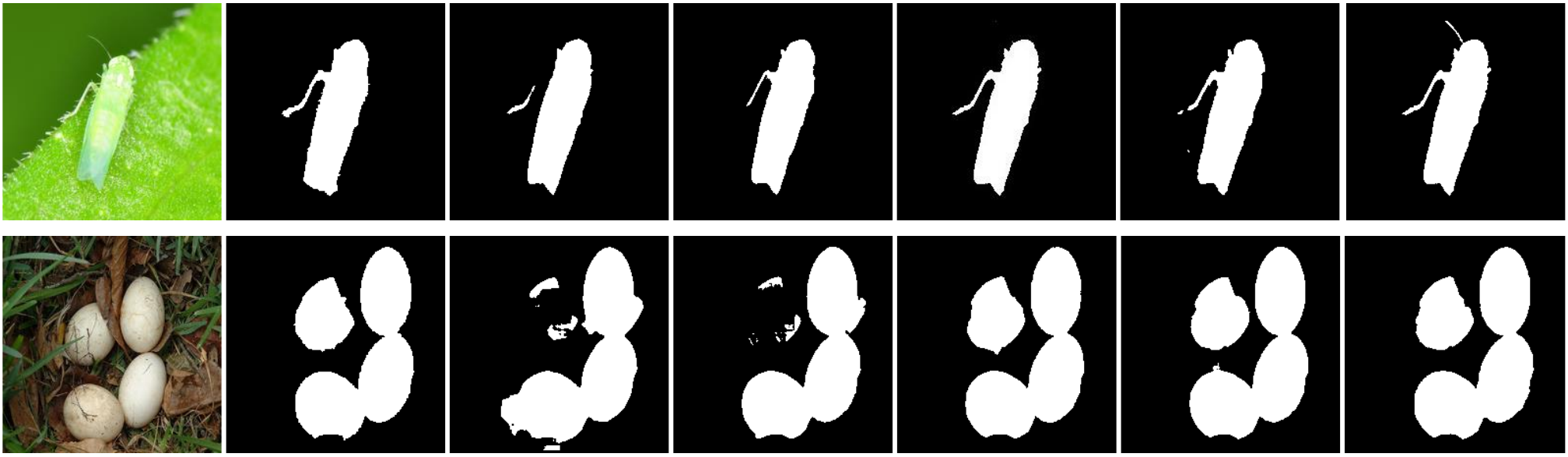}  
		\caption{Comparison results of cicada and eggs (From top to bottom) of HKU-IS between the proposed models: UNet, UNet++, DeepLabV3+, DN-I, and VM\_TUNet. The pictures from left to right are: Image; Ground Truth; and the results of  UNet, UNet++, DeepLabV3+, DN-I, and VM\_TUNet, respectively.}  
		\label{HKU1}  
	\end{figure}
	
	\begin{figure}[ht]	\centering\includegraphics[width=1\linewidth]{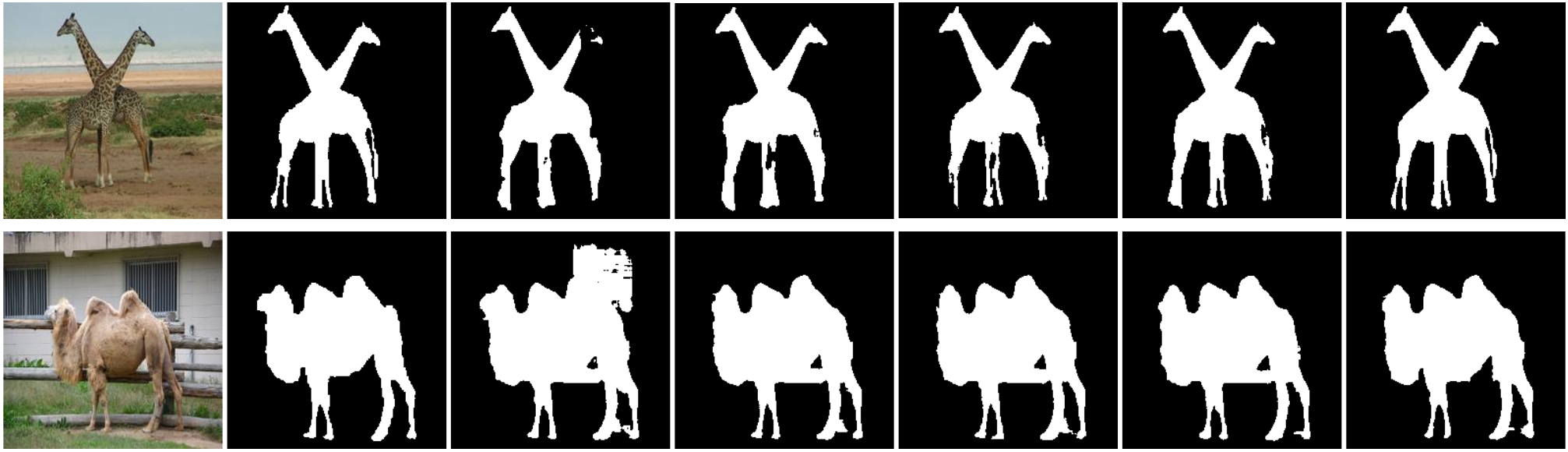}  
		\caption{Comparison results of giraffe and camel (From top to bottom) of HKU-IS between the proposed models: UNet, UNet++, DeepLabV3+, DN-I, and VM\_TUNet. The pictures from left to right are: Image; Ground Truth; and the results of  UNet, UNet++, DeepLabV3+, DN-I, and VM\_TUNet, respectively.}  
		\label{HKU2}  
	\end{figure}
	
	\paragraph{\textbf{DUT-OMRON}}
	For DUT-OMRON,  which
	contains 5168 high-quality natural images where each image contains one or more salient objects with varied and cluttered backgrounds, we resize all images to the size 256$\times$256, using 3500 images for training and 1668 images for testing. The partial results of HKU-IS is shown in Figure~\ref{DUT}. 
	
	\begin{figure}[ht]	\centering\includegraphics[width=1\linewidth]{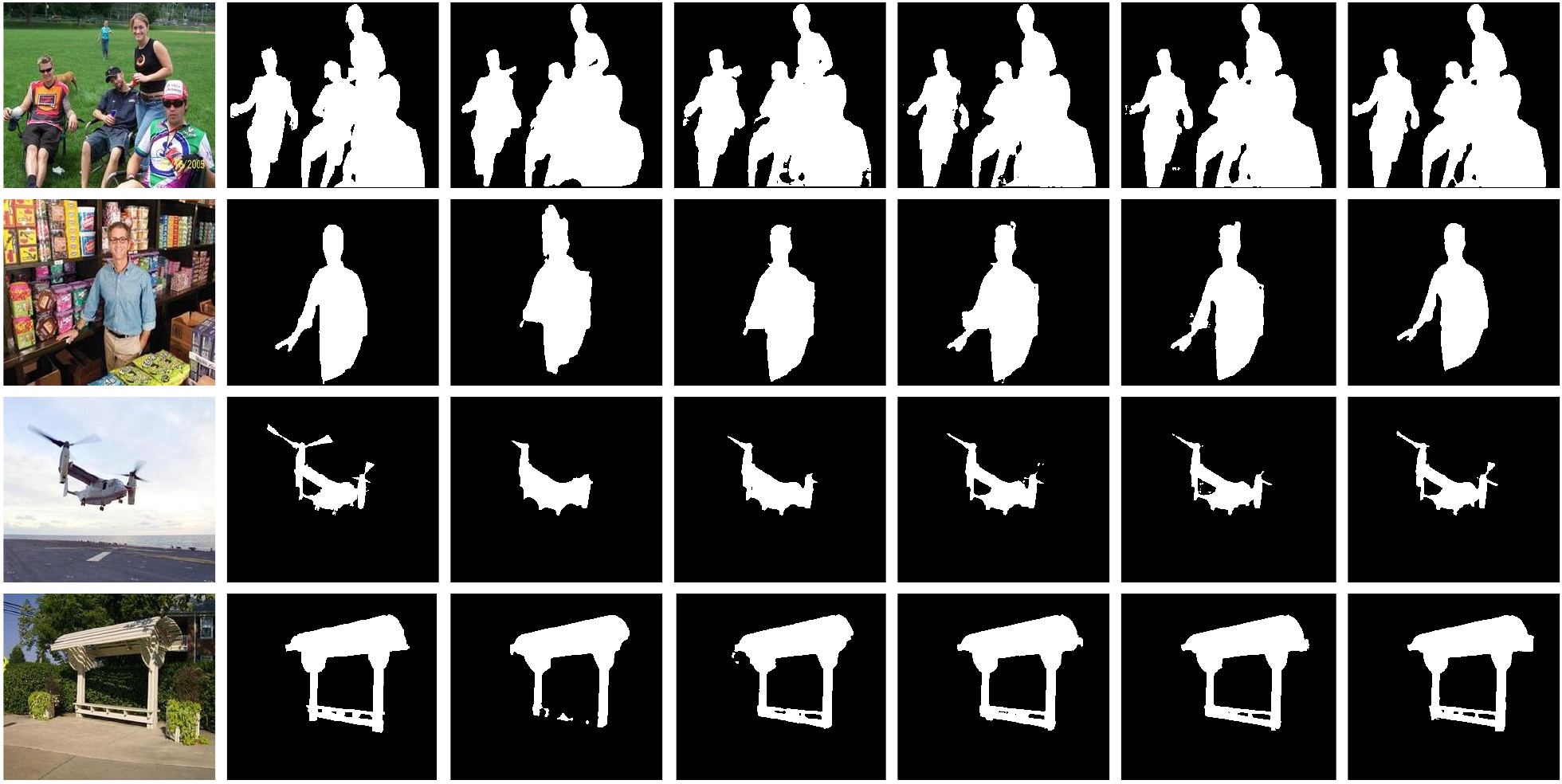}  
		\caption{Comparison results of family, gentleman, plane, and bench (From top to bottom) of DUT-OMRON between the proposed models: UNet, UNet++, DeepLabV3+, DN-I, and VM\_TUNet. The pictures from left to right are: Image; Ground Truth; and the results of  UNet, UNet++, DeepLabV3+, DN-I, and VM\_TUNet, respectively.}  
		\label{DUT}  
	\end{figure}
	
	\subsection{Comparative trial}

	\paragraph{\textbf{Replacement of VM\_TUNet with traditional variational method}}	{\color{blue}
	 Firstly, the term $F\left(f\right)$ remains indispensable in the application of Cahn-Hilliard type equations to the field of image segmentation, even without approximation via UNet. The results of traditional variational methods applied to the original equation can be found, such as in \cite{yang2019image}, as also seen in our paper with similar examples, such as the tiger segmentation we present in Figure~\ref{crane-tiger}. As can be seen from the Figure~\ref{CH}, even when ignoring the influence of the grass in the results of the traditional variational numerical method by disregarding the mask, meaning we do not consider the grass portion in its image result, it is still clearly visible that the segmentation effect of VM\_TUNet on the tiger's tail and neck is significantly better than that of the traditional method.
	\begin{figure}
		\centering\includegraphics[width=1\linewidth]{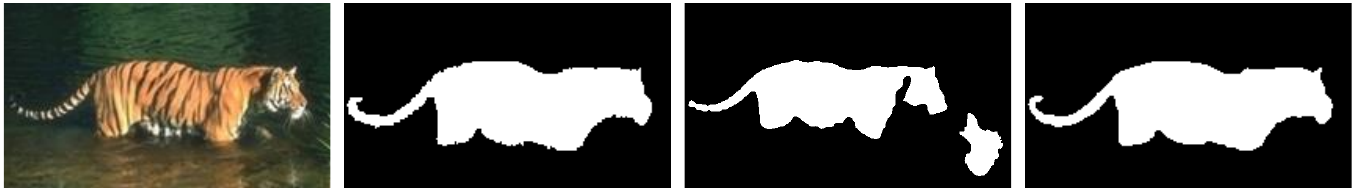}  
		\caption{Comparison results of the tiger of ECSSD between traditional method and VM\_TUNet. The pictures from left to right are: Image; Ground Truth; the results of traditional method and VM\_TUNet respectively.}
		\label{CH}   
	\end{figure}
}
	\paragraph{\textbf{Replacement of UNet with a simple CNN}}\label{CNN}	To examine whether the performance gain of VM\_TUNet is due to the specific architecture of UNet or simply to its parameter scale, we conduct an ablation study by replacing the UNet-based approximation of $F(f)$ with a plain convolutional neural network--FlatCNN. The FlatCNN consists of a deep stack of 2D convolutional layers followed by batch normalizations and ReLU activations, without skip connections, downsampling and upsampling modules. The total number of parameters is controlled to be approximately 30 million, comparable to the UNet used in the main experiments. The architecture follows the following form: $\text{Conv-BN-ReLU}\times N\rightarrow \text{Conv} \rightarrow \text{Sigmoid}$ where $N$ is the number of convolutional blocks, adjusted to match the UNet's parameter budget, whose architecture overview is shown in Table~\ref{table4}.
	
	\begin{table}[ht]
		\caption{Architecture Design of FlatCNN.}
		\label{table4}
		\centering
		\resizebox{0.98\textwidth}{!}{
			\begin{tabular}{ccccc}
				\toprule
				Hierarchy & Layer           & \multicolumn{1}{l}{Number of output channel} & \multicolumn{1}{l}{Convolution kernel size} & \multicolumn{1}{l}{Activation \& Normalization} \\
				\midrule
				1   & Conv2D & 128 & 3$\times$3 & ReLU+BN  \\        
				2-5 & Conv2D$\times$4 & 256 & 3$\times$3 & ReLU+BN \\
				6-10 & Conv2D$\times$5 & 512 & 3$\times$3 & ReLU+BN  \\
				11  & Conv2D & 256 & 3$\times$3 & ReLU+BN     \\
				12  & Conv2D & 128 & 3$\times$3 & ReLU+BN      \\
				13  & Conv2D & 1   & 1$\times$1 & Sigmoid    \\
				\bottomrule
		\end{tabular}}
	\end{table}
	
	At the same time, we also use ResNet50 and DenseNet264 whose number of parameters is similar to UNet to approximate $F(f)$ and we have some selected segmentation results in Figure~\ref{Four} and Table~\ref{table5} under the same experimental conditions where the result shows that UNet performs better than simple CNNs.
	
	\begin{figure}[ht]	\centering\includegraphics[width=1.0\linewidth]{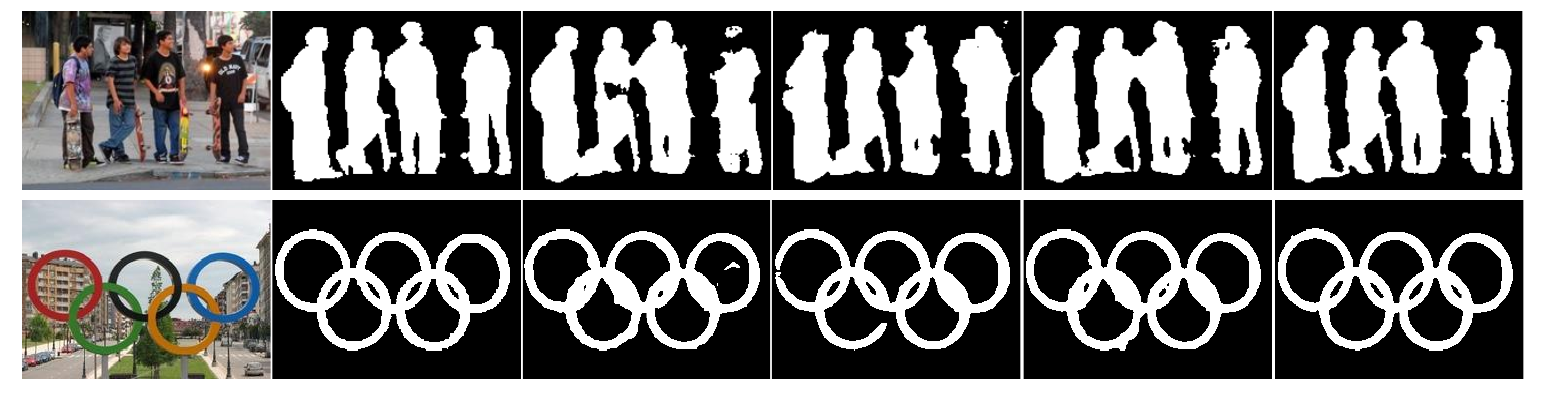} 
		\caption{Comparison results of boys (Above) and Olympic rings (Below) of ECSSD between models trained by ResNet50, DenseNet, FlatCNN and UNet. The pictures from left to right are: Image; Ground Truth; the results trained by ResNet50, DenseNet, FlatCNN and UNet, respectively.} 
		\label{Four}  
	\end{figure}
	
	\begin{table}
		\caption{Comparison of the Accuracy and Dice score of models trained by ResNet50, DenseNet, FlatCNN and UNet, respectively on ECSSD, Extended Complex Scene Saliency Dataset.}
		\label{table5}
		\centering
		\begin{tabular}{llll}
			\toprule
			&\multicolumn{2}{c}{ECSSD}               \\
			\cmidrule(r){2-3}  &   Accuracy   & Dice score  \\
			\midrule
			\multicolumn{1}{c}{ResNet50}& \multicolumn{1}{c}{0.823$\pm$0.003}& \multicolumn{1}{c}{0.801$\pm$0.001}&\\ 
			\multicolumn{1}{c}{DenseNet}& \multicolumn{1}{c}{0.855$\pm$0.001}& \multicolumn{1}{c}{0.831$\pm$0.001}&\\ 	
			\multicolumn{1}{c}{FlatCNN}& \multicolumn{1}{c}{0.888$\pm$0.001}& \multicolumn{1}{c}{0.872$\pm$0.002}& \\ 
			\multicolumn{1}{c}{UNet}& \multicolumn{1}{c}{\textbf{0.949}$\pm$0.002}& \multicolumn{1}{c}{\textbf{0.906}$\pm$0.002}&  \\ 
			\bottomrule
		\end{tabular}
	\end{table}
	
	\paragraph{\textbf{Replacement of TFPM with finite difference method}}	For Eq.~\eqref{v2} and Eq.~\eqref{u2}, we just use simple finite difference method (FDM) without TFPM to get the procedure $u^n\rightarrow v^n\rightarrow u^{n+1}$ where $\Delta v^n$ and $\Delta u^n$ are only approximated by central difference, which are realized by convolution $W_{\Delta}*v^n$ and $W_{\Delta}*u^n$ by Eq.~\eqref{laplace}. The computational schemes are
	\begin{equation}\label{v3}
		\begin{cases} 
			v^n=\varepsilon_{1}W_{\Delta}*u^n-\dfrac{1}{\varepsilon_{2}}\big(4(u^n)^3-6(u^n)^2+2u^n\big),&\text{in}\quad\Omega, \\
			v^n(0,y)=v^n(L_1,y),&0\leq y\leq L_2,\\[0.5em]
			v^n(x,0)=v^n(x,L_2),&{ }0\leq x\leq L_1,\\
			v^0=\varepsilon_{1}W_{\Delta}*u^0-\dfrac{1}{\varepsilon_{2}}\big(4(u_0)^3-6(u_{0})^2+2u_{0}\big),&\text{in}\quad\Omega,
		\end{cases}
	\end{equation}
	and 
	\begin{equation}\label{u3}
		\begin{cases}
			u^{n+1}=u^n-\tau W_{\Delta}*v^n-\tau F(f),&\text{in}\quad\Omega,\\
			u^{n+1}(0,y)=u^{n+1}(L_1,y),&0\leq y\leq L_2,\\
			u^{n+1}(x,0)=u^{n+1}(x,L_2),&0\leq x\leq L_1,\\
			u^0=u_0,&\text{in}\quad\Omega,
		\end{cases}
	\end{equation}
	respectively. Some selected segmentation results are shown in Figure~\ref{fd} under the same experimental conditions where the result obviously shows that TFPM performs much better than FDM.
	
	\begin{figure}[ht]	\centering\includegraphics[width=1.0\linewidth]{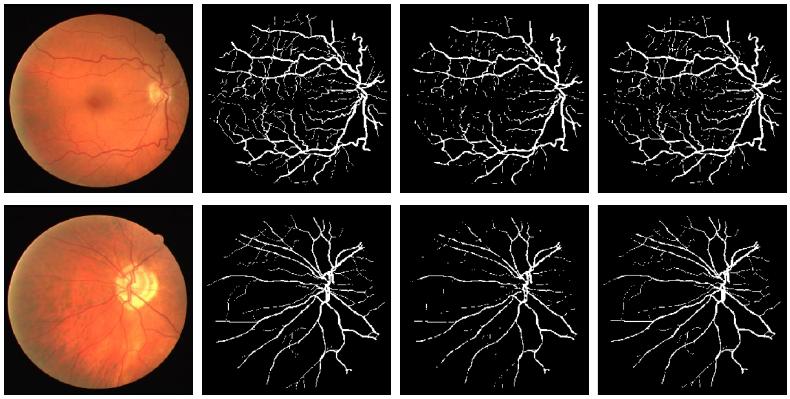} 
		\caption{Comparison results of two examples of RITE between models {\color{blue} applying} FDM and TFPM. The pictures from left to right are: Image; Ground Truth; the results of applying FDM and TFPM, respectively.} 
		\label{fd}  
	\end{figure}
	
	\subsection{Effects of hyperparameters of VM\_TUNet}
	We investigate how the hyperparameters of VM\_TUNet influence performance on the ECSSD dataset. In addition to the neural network’s weight parameters, VM\_TUNet involves several key hyperparameters: the number of blocks $M$, the time step $\tau$, and the coefficients $\varepsilon_1$ and $\varepsilon_2$, while keeping the network architecture fixed as $\boldsymbol{c}=[128,128,128,128,256]$. To provide a clearer comparison, we display results from the 250th to the 450th epoch, plotting the loss, accuracy, and dice score at intervals of 20 epochs.
	
	For the number of blocks, we test $M=1,10,40$ and fix $\tau=0.5,\varepsilon_1=1,\varepsilon_2=1$. The comparisons of the loss, accuracy, and dice score are shown in Figure~\ref{blocknumber}. We can see that VM\_TUNet with $M=10$ provides the best results. It is worth noting that increasing the number of blocks leads to a growth in model parameters, yet it greatly raises the complexity of the network. When $M$ is set too small, VM\_TUNet becomes overly simple and struggles to perform the segmentation task effectively. Conversely, an excessively large $M$ results in an overly complex model, which is more prone to getting trapped in local minima during training.
	
	\begin{figure}[ht]	\centering\includegraphics[width=1\linewidth]{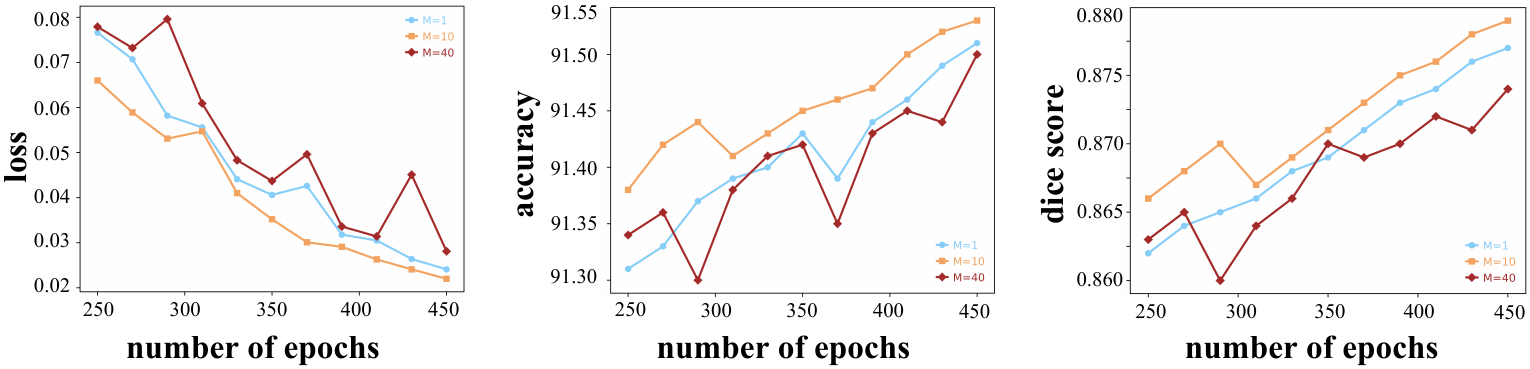}  
		\caption{Effect of the number of blocks $M$ for VM\_TUNet.}  
		\label{blocknumber}  
	\end{figure}
		
	We next fix $M=10,\varepsilon_1=1,\varepsilon_2=1$, and  evaluate $\tau=0.05,0.5,\text{ and }5$. The outcomes are presented in Figure~\ref{steptime}. In our framework, $\tau$ represents the time step: a smaller $\tau$ causes the solution to progress slowly, whereas a larger $\tau$  can lead to instability in the algorithm.
	
	\begin{figure}[ht]	\centering\includegraphics[width=1\linewidth]{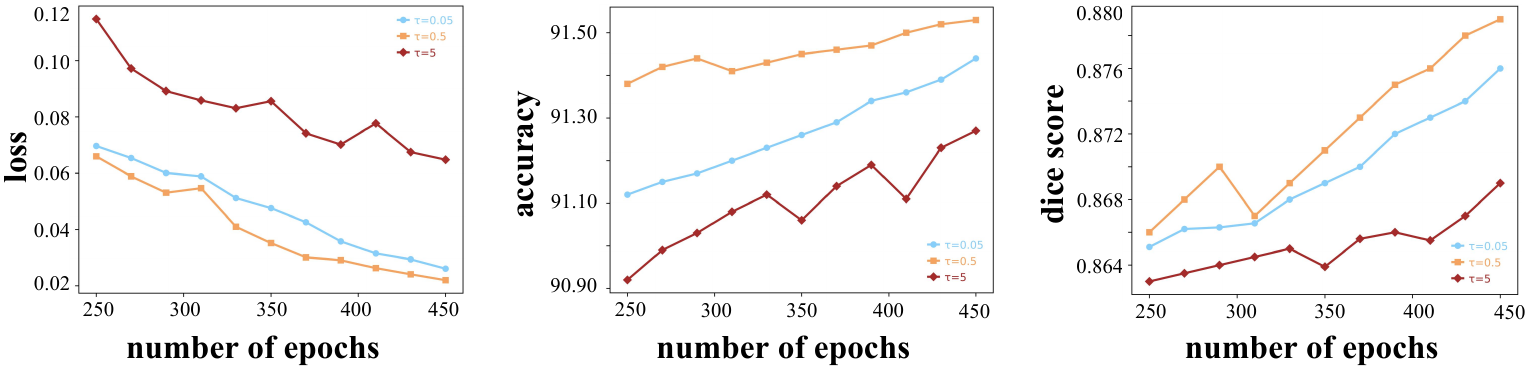}  
	\caption{Effect of the time step $\tau$ for VM\_TUNet.}  
	\label{steptime}  
	\end{figure}
	
	The next two tests are for the hyperparameters $\varepsilon_1$ and $\varepsilon_2$. We fix $M=10$ and $\tau=0.5$, and test $\varepsilon_1=0.1,1,10$ and $\varepsilon_2=0.1,1,10$, respectively. The results are shown in Figure~\ref{epsilon1}~and~\ref{epsilon2}. Hyperparameters $\varepsilon_1$ and $\varepsilon_2$ mainly determine the evolution of $u$. In order to help connect the discontinuous broken parts of the image, it is necessary to select larger values for these two parameters to drive the propagation of $u$ in those areas close to $1$. However, when $\varepsilon_1$ and $\varepsilon_2$ are selected to be large, it will result in a large transition layer, thereby blurring the boundaries of the image that is desired to be segmented. At the same time, based on the characteristics of the Cahn-Hilliard equation, when they are selected too small, although the transition layer can be made thinner, which helps to locate the boundary of target segmentation, the time evolution will be slower.

	\begin{figure}[ht]	\centering\includegraphics[width=1\linewidth]{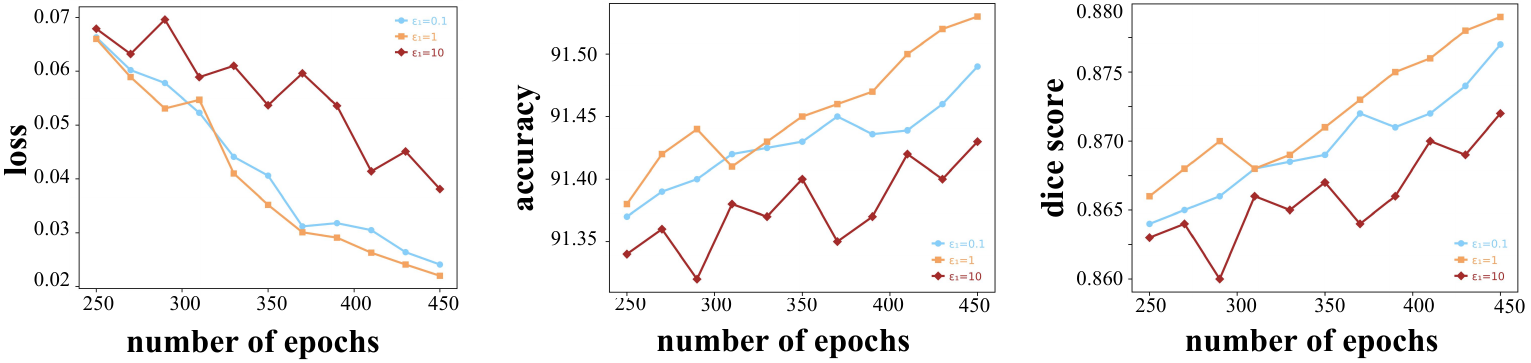}  
	\caption{Effect of the hyperparameter $\varepsilon_{1}$ for VM\_TUNet.}  
	\label{epsilon1}  
	\end{figure}
	
	\begin{figure}[ht]	\centering\includegraphics[width=1\linewidth]{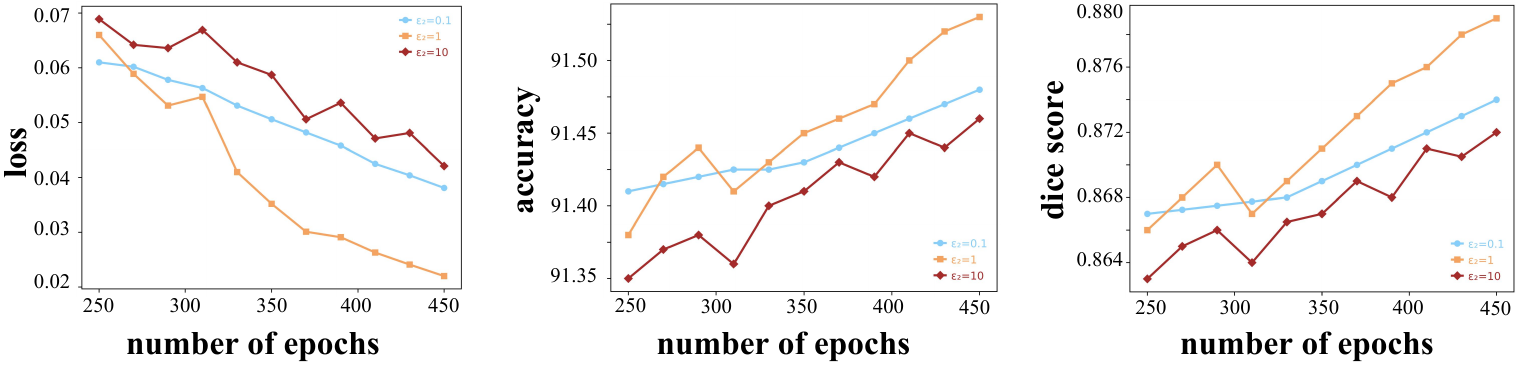}  
	\caption{Effect of the hyperparameter $\varepsilon_{2}$ for VM\_TUNet.}  
	\label{epsilon2}  
	\end{figure}

	\section{Conclusion}\label{Conclusion}
	In this paper, we propose VM\_TUNet, an innovative approach that integrates the strengths of both traditional variational models and modern deep learning methods for image segmentation. By leveraging the fourth-order Cahn-Hilliard equation and combining it with a deep learning framework based on the UNet architecture, we aim to overcome the limitations of conventional segmentation techniques. The key innovation in VM\_TUNet lies in seamlessly incorporating variational principles, which provides interpretability and theoretical rigor, into the flexibility and accuracy of deep learning models. By using a data-driven operator, we eliminate the need for manual parameter tuning, which is a common challenge in traditional variational methods. Next, our use of the TFPM enhances boundary sharpness and structural consistency, which is particularly valuable for semantic segmentation tasks requiring precise delineation. 
	
	Our experimental results demonstrate that VM\_TUNet outperforms existing CNN-like segmentation models like UNet, UNet++, DeepLabV3+, and DN-I, which achieving higher accuracy and dice scores, particularly in challenging segmentation tasks with sharp boundaries. On the other hand, our proposed VM\_TUNet model is primarily designed with an emphasis on lightweight architecture and strong mathematical interpretability, rather than maximizing empirical performance through large-scale models. Due to constraints in computational resources, we focus on models that can be trained and evaluated efficiently without the need for extensive pretraining or large-scale infrastructure. Furthermore, the core contribution of this work lies in the integration of variational PDE-based modeling with data-driven deep learning to enhance boundary preservation and theoretical rigor. As such, our experimental comparisons are centered around representative models within a similar design scope rather than squeezing out engineering performance.
	
	\section*{Acknowledgments}
	The authors are grateful to the reviewers for the constructive comments and
	valuable suggestions which have improved the paper. This research was supported by National Natural Science Foundation of China (NSFC No.12025104).

	\bibliographystyle{abbrv}
	\bibliography{references} 


\end{document}